\documentclass[letterpaper, 10 pt, conference]{ieeeconf}  

\IEEEoverridecommandlockouts                              

\title{\LARGE \bf
A Closed-Form CLF-CBF Controller for Whole-Body\\ 
Continuum Soft Robot Collision Avoidance
}

\author{Kiwan Wong$^{1}$, Maximillian Stölzle$^{1}$, Wei Xiao$^{2, 1}$, Daniela Rus$^{1}$
\thanks{
    $^{1}$Computer Science and Artificial Intelligence Laboratory, Massachusetts Institute of Technology, Cambridge, MA 02139, USA {\tt\scriptsize \{ kiwan588, mstolzle, weixy, rus \}@mit.edu}.
    $^{2}$Robotics Engineering Department, Worcester Polytechnic Institute, Worcester, MA 01609, USA {\tt\scriptsize wxiao3@wpi.edu}.
}
\thanks{The work by K.~Wong was supported by The Hong Kong Jockey Club Scholarships. We are grateful to the Singapore–MIT Alliance for Research and Technology (SMART) M3S for support of this work.}
}

\usepackage{xcolor}
\usepackage{siunitx}
\usepackage{subcaption}
\usepackage{booktabs}
\usepackage{algorithm}
\usepackage{algorithmic}
\usepackage{graphicx}
\usepackage{amsmath,amsfonts,bm,mathtools} 
\usepackage{fix-cm}
\usepackage{url}
\usepackage[capitalise,nameinlink]{cleveref}

\usepackage{amsthm}

\newtheoremstyle{runindef}
  {}                      
  {}                      
  {\normalfont}           
  {}                      
  {\itshape}              
  {:}                     
  { }                     
  {\thmname{#1}\ \thmnumber{#2}\thmnote{\,(#3)}} %
\theoremstyle{runindef}
\newtheorem{definition}{Definition}

\newtheoremstyle{runinthm}
  { }               
  { }               
  {\normalfont}     
  { }               
  {\bfseries}       
  {:}               
  { }               
  {\thmname{#1}\ \thmnumber{#2}\thmnote{\,(#3)}}

\theoremstyle{runinthm}

\newtheorem{lemma}{Lemma}

\renewcommand{\baselinestretch}{0.867}

\crefname{figure}{Fig.}{Figs.}      
\Crefname{figure}{Fig.}{Figs.}     
\usepackage{acro}
\newcommand{\newac}[2]{\DeclareAcronym{#1}{short=#1,long=#2}}
\newac{APF}{Artificial Potential Field}
\newac{CBF}{Control Barrier Function}
\newac{CLF}{Control Lyapunov Function}
\newac{CC}{Constant Curvature}
\newac{CS}{Constant Strain}
\newac{COM}{Center of Mass}
\newac{DCSAT}{Differentiable Conservative SAT}
\newac{DCM}{Discretized Cosserat Rod Model}
\newac{DOF}{Degrees of Freedom}
\newac{EOM}{Equations of Motion}
\newac{GVS}{Geometric Variable Strain} 
\newac{HOCBF}{High-Order CBF}
\newac{HOCLF}{High-Order CLF}
\newac{JIT}{Just-in-Time}
\newac{LNN}{Lagrangian Neural Network}
\newac{LSE}{Log-Sum-Exp}
\newac{LSTM}{Long Short-Term Memory}
\newac{MAE}{Mean Absolute Error}
\newac{ML}{Machine Learning}
\newac{MSE}{Mean Squared Error}
\newac{MPC}{Model Predictive Control}
\newac{PAC}{Piecewise Affine Curvature}
\newac{PCS}{Piecewise Constant Strain}
\newac{PCC}{Piecewise Constant Curvature}
\newac{PDE}{Partial Differential Equation}
\newac{QP}{Quadratic Program}
\newac{RL}{Reinforcement Learning}
\newac{RMSE}{Root Mean-Squared Error}
\newac{SAT}{Separating Axis Theorem}
\newac{SSAT}{Smooth SAT}
    
\begin{document}
\bstctlcite{IEEEexample:BSTcontrol}

\maketitle
\thispagestyle{empty}
\pagestyle{empty}

\begin{abstract}

%
%
Safe operation is critical for deploying robots in human-centered 3D environments. Soft continuum manipulators offer passive safety through compliance but require active control to ensure reliable collision avoidance. Existing methods, such as sampling-based planning, are computationally expensive and lack formal safety guarantees, limiting real-time whole-body collision avoidance. This paper presents a closed-form Control Lyapunov Function–Control Barrier Function (CLF–CBF) controller that enforces real-time, 3D obstacle avoidance without online optimization. By analytically embedding safety constraints into the control input, the method guarantees stability and safety under the stated modeling assumptions, improving optimality guarantees, avoiding feasibility issues commonly encountered in online optimization-based approaches, and being up to 10x faster than common online optimization-based approaches for solving the CLF-CBF QP and up to 100x faster than traditional sampling-based planning approaches. Simulation and hardware experiments on a tendon-driven soft manipulator demonstrate accurate 3D trajectory tracking and robust avoidance in cluttered environments. The proposed framework enables scalable, provably safe control for soft robots operating in dynamic, safety-critical settings.
\end{abstract}

\section{Introduction}
Deploying robots in human-centered environments, such as manufacturing, healthcare, and assistive applications~\cite{hall2019acceptance}, demands not only demonstrable safety~\cite{stolzle2025phdthesis} but also user trust in the robot’s ability to operate predictably and avoid hazardous collisions. Soft continuum manipulators offer inherent safety advantages by embedding compliance directly into their physical structure~\cite{rus2015design}. Nevertheless, as the field advances toward higher precision, larger payloads, and more dynamic tasks, modern soft robot designs increasingly incorporate greater stiffness~\cite{bruder2023increasing}, stronger actuation~\cite{tang2020leveraging}, and hybrid rigid–soft architectures~\cite{zuo2025umarm}. Under these conditions, passive compliance alone is insufficient: soft manipulators can still produce significant impact forces, suffer structural damage from excessive deformation, or experience instability during environmental contact. Consequently, formal, proactive obstacle avoidance mechanisms are essential to ensure reliable and safe operation in safety-critical scenarios.

\begin{figure}[htbp]
    \centering
    \includegraphics[width=1.0\linewidth]{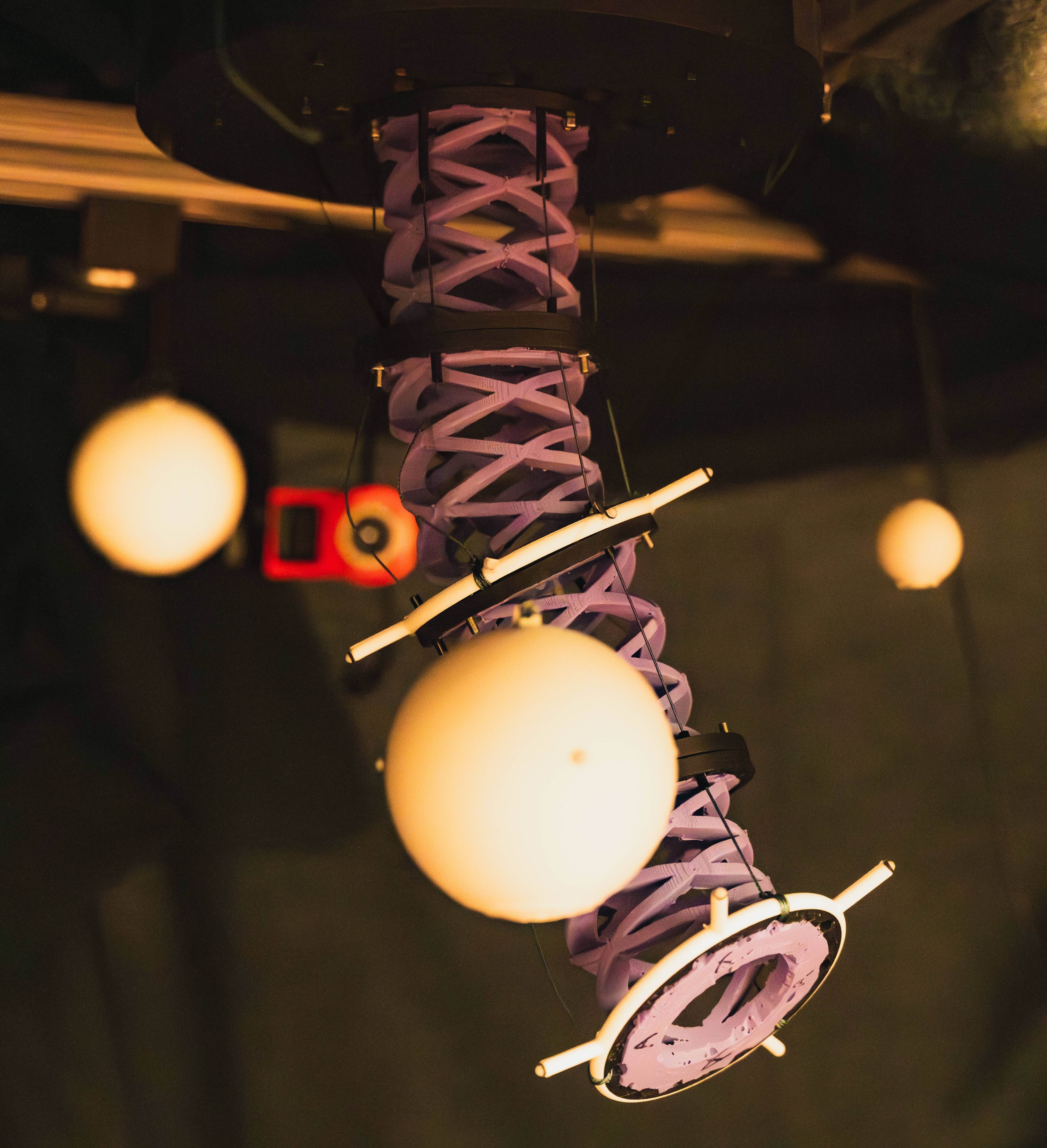}
    \caption{Experimental setup of the tendon-driven soft robotic manipulator used to validate the proposed closed-form CLF–CBF controller. The two-segment robot is mounted to a ceiling fixture and operates in an environment with spherical obstacles. Reflective markers are attached to each segment for real-time motion capture and feedback control during safe trajectory tracking experiments.}
    \label{fig:enhanced-setup}
\end{figure}

Although the distinctive characteristics and complex deformations of continuum soft robots necessitate specialized techniques, the literature on obstacle avoidance for 3d continuum soft robots remains limited, largely focusing on path- or trajectory-planning formulations tackled with sampling-based planners~\cite{torres2015motion, kuntz2017motion, meng2022rrt, li2024s, lang2025obstacle}, such as RRT*~\cite{karaman2011sampling}, and \ac{MPC}-based motion planning~\cite{spinelli2022unified, hachen2025non}. Related work on general path and motion planning for continuum soft robots is likewise sparse and tends to emphasize search-based~\cite{rao2024towards, ouyoucef2025durability} and sampling-based methods~\cite{du2025physics}.
However, search-based planners typically require discretizing the task or configuration space—sacrificing precision—and running sampling-based planners online can be computationally expensive. Moreover, because these methods generally do not produce actuation setpoints, they often necessitate complex cascaded architectures with high-rate inner-loop controllers.  Classical reactive approaches like potential field methods, have also been explored in other robotic settings; however, they typically lack formal safety guarantees.

A promising alternative direction is to enforce safety online using \ac{CBF}-based constraints. This typically requires solving a \ac{QP} in real time, either formulated as a \ac{CBF} safety filter on a nominal controller~\cite{patterson2024safe, xu2024hybrid, dickson2025safe} or by jointly optimizing task progress through \ac{CLF}-based objectives in a CLF-CBF framework~\cite{wong2025contact}. However, the continuous deformability of soft robots poses inherent challenges to such online optimization: accurately modeling their kinematics and dynamics often demands many degrees of freedom (e.g., many segments in PCC/PCS models)~\cite{della2023model, armanini2023soft}, and faithful collision detection via geometric primitives (e.g., spheres or convex polygons) requires fine-grained discretization, yielding a large number of safety constraints. Combined, these factors render solving the optimization problem online challenging, limit attainable control frequencies or force model or constraint simplifications, thereby weakening formal safety guarantees. Moreover, optimization-based approaches can suffer from numerical instability, suboptimality when tight runtime budgets preclude global search, and infeasibility when multiple constraints interact—a common scenario in whole-body avoidance. As a result, current methods either lack real-time performance, fail to scale to high-dimensional soft bodies, or cannot guarantee formal whole-body safety.

In this paper, we address these limitations by introducing a closed-form CLF-CBF controller for whole-body obstacle avoidance of soft continuum robots.
Unlike traditional \ac{QP}-based formulations, the proposed control law directly embeds the \ac{CBF} constraints into the control input through a closed-form expression, avoiding online optimization entirely.
This enables real-time safety enforcement even with a large number of safety constraints, making it particularly suitable for high-dimensional soft robots operating in cluttered environments.
Furthermore, the closed-form structure improves numerical stability and eliminates feasibility issues commonly encountered in optimization-based approaches.

The main contributions of this work are:
\begin{itemize}
\item A closed-form CLF--CBF controller for high-dimensional continuum soft robots that ensures real-time whole-body obstacle avoidance under dense collision constraints, without requiring online QP solving, addressing a regime where existing explicit CBF methods for rigid or low-dimensional systems do not readily apply.
\item Extensive simulation and real-world hardware experiments results demonstrating real-time performance, formal safety guarantees, and accurate task-space tracking in cluttered environments.
\item Comparison of the proposed closed-form CLF-CBF controller against an obstacle avoidance baseline strategy consisting of an RRT* planner and a kinematic, obstacle-unaware low-level controller on a setpoint regulation task in simulation. 
\end{itemize}

\noindent The code is publicly available on GitHub.\footnote{\scriptsize \url{https://github.com/KWWnoob/closed-form-clf-cbf-soft-robot}}
\section{Background}
This section introduces the background necessary for introducing the methodology.

\subsection{Kinematics of a Tendon-Driven Soft Robotic Arm}
We model the kinematics of a soft robot moving in 3D space via the \ac{PCS}~\cite{renda2018discrete} formulation, which approximates the continuous backbone by discretizing it into $N$ segments, where each segment exhibits spatially constant strain $\bm{\xi}_i$
\begin{equation}
    \bm{\xi}_i = 
    \begin{bmatrix}
        \kappa_{x,i} & \kappa_{y,i} & \kappa_{z,i} & \sigma_{x,i} & \sigma_{y,i} & \sigma_{z,i}
    \end{bmatrix}^\top
    \in \mathbb{R}^6,
\end{equation}
where $\kappa_{y,i}, \kappa_{z,i}$ represent the bending strains, $\kappa_{x,i}$ is the torsional strain around the backbone axis, $\sigma_{x,i}$ denotes the axial/elongation strain, and $\sigma_{y,i},\sigma_{z,k}$ the shear strains. The overall configuration vector $\bm{q} \in \mathbb{R}^{n_\mathrm{q}}$ that describes the deviation from the reference strain $ \bar{\bm{\xi}} \in \mathbb{R}^{n_\mathrm{q}}$ is defined as
\begin{equation}
    \bm{q} = 
    \begin{bmatrix}
    \bm{\xi}_1^\top & \cdots & \bm{\xi}_N^\top
    \end{bmatrix}^\top - \bar{\bm{\xi}}
    \in \mathbb{R}^{n_\mathrm{q}},
\end{equation}
where $n_\mathrm{q} = 6N$ in case all strains are active.
The corresponding positional forward kinematics map
\(\bm{p} = \mathrm{FK}(\bm{q}, s): \mathbb{R}^{n_\mathrm{q}} \times (0,L] \to \mathbb{R}^3 \)
returns the position $\bm{p}$ at the backbone abscissa $s \in (0, L]$, where $L$ is the nominal/reference length of the soft robot.

Next, we introduce a model that relates the actuation (i.e., the tendons) to the soft robot strains/configurations.
\paragraph{Tendon Length Model}
Consider $m$ tendons placed on the cross-section at a constant radius $R$, each at a polar angle $\phi_j \in [-\pi, \pi)$ in the local cross-sectional y-z plane.
Then, the tendon distance from the backbone in the local frame takes the form $\bm{d}_{\mathrm{t},j} = [0, \cos(\phi_j) R_j, \sin(\phi_j) R_j ]^\top$. Let $\widehat{\cdot}: \mathbb{R}^6 \to \mathfrak{se}(3)$ denote the SE(3) hat/wedge operator and $\tilde{\cdot}: \mathbb{R}^3 \to \mathfrak{so}(3)$ the tilde operator returning the skew matrix~\cite{renda2020geometric}.
Then, the corresponding tangent vector $\bm{t}_{\mathrm{t},j,i}(\bm{q}_i)$ and the local actuation mapping $\bm{\Theta}_{\mathrm{t},j,i}(\bm{q}_i)$ of the $j$th tendon in the $i$th segment can be constructed as~\cite{renda2020geometric}
\begin{equation}
    \bm{t}_{\mathrm{t},j,i}(\bm{q}_i) = \frac{[\widehat{\bm{\xi}}_i \, \bm{d}_{\mathrm{t},j}]_3}{\lVert \widehat{\bm{\xi}}_i \, \bm{d}_{\mathrm{t},j} \rVert_2},
    \quad
    \bm{\Theta}_{\mathrm{t},j,i}(\bm{q}_i) = \begin{bmatrix}
        \tilde{\bm{d}}_{\mathrm{t},j} \, \bm{t}_{\mathrm{t},j,i}(\bm{q}_i)\\
        \bm{t}_{\mathrm{t},j,i}(\bm{q}_i)
    \end{bmatrix} \in \mathbb{R}^6.
\end{equation}
Referring to the length of the $i$th segment as $L_i$, the tendon length in that segment can be defined uniquely (neglecting tendon slack) as~\cite{pustina2024input}
\begin{equation}\label{eq:tendon_length}
    \ell_{j,i}(\bm{q}_i) 
    = \int_{0}^{L_i} \bm{\Theta}_{\mathrm{t},j,i}^\top(\bm{q}_i) \, \bm{\xi}_i  \, \mathrm{d}s = \bm{\Theta}_{\mathrm{t},j,i}^\top(\bm{q}_i) \, \bm{\xi}_i \, L_i \in \mathbb{R}.
\end{equation}
Referring to the full length of the $j$th tendon as $\ell_{j} = \sum_{i} \ell_{j,i} \in \mathbb{R}$, stacking the tendon lengths across all segments yields
\[
\bm{\ell}(\bm{q}) 
= \bigl[\ell_{1}(\bm{q}) \ \dots \ \ell_{m}(\bm{q})\bigr]^\top
\in \mathbb{R}^m,
\]
and differentiating w.r.t. $\bm{q}$ provides the tendon-actuation Jacobian $\bm{J}_{\bm{\ell}}(\bm{q}) = \frac{\partial \bm{\ell}(\bm{q})}{\partial \bm{q}} \in \mathbb{R}^{m \times n_\mathrm{q}}$ that maps generalized velocities to tendon velocities as $\dot{\bm{\ell}} = \bm{J}_{\bm{\ell}}(\bm{q}) \, \dot{\bm{q}}$, whose elements are given by~\cite{pustina2024input}
\begin{equation}
    \bm{J}_{\bm{\ell},j,i}(\bm{q}_i) = \int_{0}^{L_i} \bm{\Theta}_{\mathrm{t},j,i}^\top(\bm{q}_i) \, \mathrm{d}s = \bm{\Theta}_{\mathrm{t},j,i}^\top(\bm{q}_i) \, L_i.
\end{equation}
Please note that the actuation matrix $\bm{A} \in \mathbb{R}^{n_\mathrm{q} \times m}$ that maps tendon forces to generalized torques is the transpose of the tendon-actuation Jacobian, i.e., $\bm{A}(\bm{q}) = \bm{J}_{\bm{\ell}}^\top(\bm{q})$~\cite{pustina2024input}.



\paragraph{Differential Actuation Kinematics}
Let the control input be the tendon length rates $\bm{u} := \dot{\bm{\ell}} \in \mathbb{R}^m$. 
Defining the system state as $\bm{x} := \bm{q}$ with $n := n_\mathrm{q}$\footnote{Please note that we will in the following use $\bm{x}$ and $\bm{q}$ interchangeably for referring to the soft robot state.} while neglecting the inertial dynamics, we can describe the evolution of the configuration using the differential actuation kinematics
\begin{equation}\label{eq:control-affine-ode}
    \dot{\bm{q}} = \bm{J}_{\bm{\ell}}^+(\bm{q}) \, \dot{\bm{\ell}} = \underbrace{\bm{A}^{\top^+}(\bm{q})}_{g(\bm{x})} \, \bm{u}.
\end{equation}
The control-affine standard form of the ODE is then given by $\dot {\bm{x}} = f(\bm x) + g({\bm{x}})\bm{u}$ with $f(\bm x)=\bm{0}_n$, and $g:\mathbb{R}^m \rightarrow \mathbb{R}^{m \times d}$ is locally Lipschitz.

\subsection{Control Barrier Functions and Control Lyapunov Functions}
Given a barrier function $b(\bm{x}): \mathbb{R}^{n} \to \mathbb{R}$ and a Lyapunov function $V(\bm{x}): \mathbb{R}^{n} \to \mathbb{R}$, both continuous in the system state $\bm{x} \in \mathbb{R}^n$, one can define \ac{CBF} and \ac{CLF} as follows~\cite{ames2019control}
\begin{definition}[Control Barrier Function \cite{ames2019control, xiao2021high}]
    Let $C = \{\bm{x} \in \mathbb{R}^n:b(\bm{x}) \geq 0\}$ be the safe set, then the continuously differentiable function $b: \mathbb{R}^n \rightarrow \mathbb{R}$ is a \ac{CBF} for system~(\ref{eq:control-affine-ode}) if there exists a class $\mathcal{K}$ function $\alpha$ such that
    \begin{equation}
        \dot{b}(\bm{x})+\alpha(b(\bm{x})) \geq 0,
    \end{equation}
    for all $\bm{x}\in C$.
\label{def:cbf}
\end{definition}

\begin{definition}[Control Lyapunov Function \cite{ames2019control}]
A continuously differentiable function $V:\mathbb{R}^n \to \mathbb{R}$ 
is a globally and exponentially stabilizing \ac{CLF} for system~(\ref{eq:control-affine-ode})
if there exist constants $c_1>0$, $c_2>0$, and $c_3>0$ such that
\begin{gather}
    c_1 \|\bm{x}\|^2 \le V(\bm{x}) \le c_2 \|\bm{x}\|^2, \\
    \inf_{\bm{u}\in\mathcal{U}} [ L_f V(\bm{x}) + L_g V(\bm{x}) \bm{u} + c_3 V(\bm{x}) ] \le 0,
\end{gather}
for all $\bm{x} \in \mathbb{R}^n$.
\label{def:clf}
\end{definition}

The \ac{CBF}~\eqref{def:cbf} and \ac{CLF}~\eqref{def:clf} can be integrated into a convex \ac{QP} optimization problem:
\begin{equation}
\begin{gathered}
    \min_{\bm{u},\delta} \ \|\bm{u}\|_2^2 + p\delta^2,\\
    \text{s.t.} \ \dot{V}(\bm{x},\bm{u}) + c_3 V(\bm{x}) \le \delta, \\
    \dot{b}(\bm{x},\bm{u}) + \alpha(b(\bm{x})) \ge 0.
\label{eq:controlQPGeneral}
\end{gathered}
\end{equation}

To keep the \ac{QP} feasible when several \acp{CBF} and \acp{CLF} contradict each other, we usually add a non-negative slack $\delta \ge 0$ with penalty $p > 0$. Typically, \acp{CLF} capture performance objectives and \acp{CBF} capture safety and other constraints; assigning slack to lower priority terms let the controller trade performance for safety.

\section{Closed-Form Solution for CLF-CBF Controller}
This section presents a closed-form formulation of the CLF–CBF controller for whole-body soft robot collision avoidance via dense spatial discretizations, while relying on kinematic model rather than full continuum dynamics~\cite{della2023model}. By explicitly solving the underlying \ac{QP}, we obtain an closed-form control law that guarantees both safety and convergence according to the task objective. 

\subsection{Problem Formulation}
\subsubsection{Control Barrier Functions}
We consider a soft robotic arm operating within a three-dimensional workspace $\mathcal{W} \subset \mathbb{R}^3$, populated by $N_{\mathrm{obs}}$ spherical obstacles $\mathcal{W}_{\mathrm{obs}} = \{\mathcal{O}_{1}, \dots, \mathcal{O}_{N_{\mathrm{obs}}}\}$, which we assume to be known, static, centered at inertial frame position $\bm{p}_{\mathrm{obs},j}$ and exhibiting a radius of $R_{\mathrm{obs},j}$. Although other obstacle geometries (e.g., convex polygons~\cite{wong2025contact}) can be considered in future work, we focus on spheres for clarity of exposition. 

To provide full-body collision avoidance, we approximate the robotic arm the union of a discrete chain of $N_{\mathrm{res}}$ spheres $\mathcal{R} = (S_1, \dots, S_{N_{\mathrm{res}}})$, where each sphere $S_i$ with radius $R_{\mathrm{res},i}$ is centered at inertial frame position $\bm{p}_{\mathrm{res},i} = \operatorname{FK}(\bm{q},s_i) \in \mathbb{R}^3$ corresponding to a uniformly spaced backbone abscissa $s_i$ along the robot’s backbone curve and captures the local body geometry at that location. Thus, the overall robot body is approximated as $\bigcup_{i=1}^{N_{\mathrm{res}}} S_i$.

We can then detect collisions between the $N_{\mathrm{res}}$ spheres representing the robot and the $N_{\mathrm{obs}}$ obstacles in the environment. Let the configuration-dependent distance function $d_{i,j}(\bm{q}): \mathbb{R}^{n_\mathrm{q}} \to \mathbb{R}$ between the $i$th robot sphere $S_i$ and the $j$th obstacle $\mathcal{O}_j$ be defined as
\begin{equation}
    d_{i,j}(\bm{q}) = \lVert \bm{p}_{\mathrm{obs},j} - \operatorname{FK}(\bm{q},s_i) \rVert_2 - R_{\mathrm{obs},j} - R_{\mathrm{res},i}.
\end{equation}
Then, we require all pairwise distances to be greater than a certain safety margin $d_{\mathrm{safe}}$, namely,

{\small
\begin{equation}\label{eq:collisionCBF}
    b_{i,j}(\bm{q}) := d_{i,j}(\bm{q}) - d_{\mathrm{safe}} \ge 0, \,
    \forall \: i\in \mathbb{N}_{N_{\mathrm{res}}},\
    j \in \mathbb{N}_{N_{\mathrm{obs}}}.
\end{equation}
}


\subsubsection{Control Lyapunov Function}
As an example, we consider an \ac{CLF} for positional\footnote{Please note that this could be easily extended to (tip) orientation regulation~\cite{prakash2024deep}.} operational space setpoint regulation~\cite{wong2025contact}, that aims to drive the robot tip position $\bm{p}_{\mathrm{tip}}(\bm{q}) = \operatorname{FK}(\bm{q},L) \in \mathbb{R}^3$ towards a desired Cartesian position $\bm{p}_{\mathrm{tip}}^d \in \mathbb{R}^3$. The corresponding \ac{CLF} function is defined as
\begin{equation}
    V(\bm{q}) = \lVert \bm{p}_{\mathrm{tip}}^\mathrm{d} - \operatorname{FK}(\bm{q},L) \rVert_2^2.
\label{eq:tipCLF}
\end{equation}

\subsubsection{Quadratic Program}
Summarizing the above-specified safety barriers and task objectives, the \ac{QP} problem can be written as:
\begin{equation}
\begin{gathered}
    \min_{\bm{u},\,\delta}\ \|\bm{u}\|_2^2 + p \, \delta^2, \\
    \text{s.t.}\ \dot{V}(\bm{x},\bm{u}) + c_3 V(\bm{x}) \le \delta, \\
    \dot{b}_{i,j}(\bm{x},\bm{u}) + \alpha(b_{i,j}(\bm{x})) \ge 0,\ \forall i \in \mathbb{N}_{N_{\mathrm{res}}}, \: j \in \mathbb{N}_{N_{\mathrm{obs}}},
\label{eq:controlQP}
\end{gathered}
\end{equation}
where the number of constraints without further heuristic approximations is given by the product $N_{\mathrm{res}} \, N_{\mathrm{obs}}$; the higher the collision-avoidance resolution, the $\mathcal{O}(N_{\mathrm{obs}})$ more constraints the \ac{QP} must consider.

\subsection{Aggregation of CBFs}
To handle multiple safety constraints efficiently in a closed-form expression, we first aggregate all pairwise distance functions $b_{i,j}(\bm{x})$ into a \emph{single} smooth barrier function using the \ac{LSE} approximation~\cite{molnar2023composing}. Intuitively, we are now considering the \ac{CBF} as the characteristic distance describing the minimum distance between the soft robot and all obstacles in the environment
\begin{equation}
    b_{\mathrm{LSE}}(\bm{x}) = -\frac{1}{\kappa}
    \log\!\left(
        \sum_{i=1}^{N_{\mathrm{res}}}
        \sum_{j=1}^{N_{\mathrm{obs}}}
        \exp\!\left[-\kappa\, b_{i,j}(\bm{x})\right]
    \right).
\label{eq:h_lse}
\end{equation}
The parameter $\kappa > 0$ trades off smoothness and responsiveness: larger $\kappa$ emphasizes the most critical obstacle, while smaller $\kappa$ yields smoother yet more conservative avoidance.  

\begin{lemma}[LSE barrier implies all pairwise barriers]
\label{lem:lse_implies_all}
Let $\{b_{i,j}(\bm{x})\}_{i \in \mathbb{N}_{N_{\mathrm{res}}}; j \in \mathbb{N}_{N_{\mathrm{obs}}}}$ be real-valued functions and define the Log-Sum-Exp aggregation
\begin{equation*}
    b_{\mathrm{LSE}}(\bm{x})
    :=
    -\frac{1}{\kappa}
    \log \!\left(
        \sum_{i=1}^{N_{\mathrm{res}}}
        \sum_{j=1}^{n_{\mathrm{obs}}}
        \exp\!\left[-\kappa\, b_{i,j}(\bm{x})\right]
    \right),
    \quad \kappa>0.
\end{equation*}
If $b_{\mathrm{LSE}}(\bm{x}) \ge 0$, then $b_{i,j}(\bm{x}) \ge 0$ for all $i,j$.
\end{lemma}

\begin{proof}
Let $\underline{b}(\bm{x}) := \min_{i,j} b_{i,j}(\bm{x})$. Then
\[
\sum_{i,j} \exp\!\left[-\kappa\, b_{i,j}(\bm{x})\right]
\;\ge\; \exp\!\left[-\kappa\, \underline{b}(\bm{x})\right].
\]
Applying $-\frac{1}{\kappa}\log(\cdot)$, which is order-reversing, yields
{\small
\begin{equation}
    b_{\mathrm{LSE}}(\bm{x})
    = -\frac{1}{\kappa}\log \sum_{i,j} e^{-\kappa b_{i,j}(\bm{x})}
    \;\le\; -\frac{1}{\kappa}\log e^{-\kappa \underline{b}(\bm{x})}
    = \underline{b}(\bm{x}).
\end{equation}
}
Hence $\underline{b}(\bm{x}) \ge b_{\mathrm{LSE}}(\bm{x}) \ge 0$, which implies $b_{i,j}(\bm{x}) \ge 0$ for all $i \in \mathbb{N}_{N_{\mathrm{res}}}; j \in \mathbb{N}_{N_{\mathrm{obs}}}$.
\end{proof}

\subsection{Closed-form Solution Derivation}
The optimization problem in \eqref{eq:controlQP} provides a general and flexible framework for combining stability and safety objectives. 
However, solving the \ac{QP} at every control step can be computationally expensive and challenging when dealing with high-dimensional dynamics and many nonlinear safety constraints. To address this, we derive in closed form the optimal control input $\bm{u}^*$ that minimizes the objective while satisfying both \ac{CLF} and \ac{CBF} constraints.

Using $b_{\mathrm{LSE}}$ as the unified safety constraint, the CLF–CBF QP in~\eqref{eq:controlQP} can be further simplified to a two-constraint form with a \emph{relaxed CLF constraint}:
\begin{equation}
\begin{gathered}
    \min_{\bm{u},\,\delta}\ \|\bm{u}\|_2^2 + w_{\mathrm{clf}} \, \delta^2, \\
    \text{s.t.}\ \dot{V}(\bm{x},\bm{u}) + c_3 V(\bm{x}) \le \delta, \\
    \dot{b}_{\mathrm{LSE}}(\bm{x},\bm{u}) + \alpha\big(b_{\mathrm{LSE}}(\bm{x})\big) \ge 0,\\
    \delta \ge 0.
\label{eq:qp_lse_relax}
\end{gathered}
\end{equation}

Let $\bm{a}_V^\top \bm{u} + b_V \le \delta$ and $\bm{a}_h^\top \bm{u} + b_h \ge 0$ denote the CLF and CBF constraints, respectively, where
\begin{gather*}
\bm{a}_V = L_g V(\bm{x})^\top,\quad b_V = c_3 V(\bm{x}),\\
\bm{a}_h = L_g b_{\mathrm{LSE}}(\bm{x})^\top,\quad b_h = \alpha\big(b_{\mathrm{LSE}}(\bm{x})\big).
\end{gather*}

\paragraph{Lagrangian and KKT conditions \cite{boyd2004convex}}  
Introduce multipliers $\lambda_V,\lambda_h,\lambda_\delta \ge 0$ for the CLF, CBF, and slack constraints, respectively.  
The Lagrangian is
\begin{align}
\mathcal{L}(\bm{u},\delta,\lambda_V,\lambda_h,\lambda_\delta)
&= \|\bm{u}\|_2^2 + w_{\mathrm{clf}}\delta^2 \notag + \lambda_V(\bm{a}_V^\top \bm{u} + b_V - \delta) \notag\\
&\quad + \lambda_h(-\bm{a}_h^\top \bm{u} - b_h) 
+ \lambda_\delta(-\delta).
\end{align}

Taking derivatives gives the KKT stationarity conditions:
\begin{equation}
\begin{gathered}
\partial_{\bm{u}}\mathcal{L} = 0 
\ \Rightarrow\ 
2\bm{u}^* + \lambda_V \bm{a}_V - \lambda_h \bm{a}_h = 0,\\
\partial_{\delta}\mathcal{L} = 0 
\ \Rightarrow\ 
2w_{\mathrm{clf}}\delta^* - \lambda_V - \lambda_\delta = 0.   
\label{eq:first-order}
\end{gathered}    
\end{equation}

Complementary slackness yields
{\small
\begin{equation}
\begin{gathered}
\lambda_V(\bm{a}_V^\top \bm{u}^* + b_V - \delta^*) = 0,
\:
\lambda_h(-\bm{a}_h^\top \bm{u}^* - b_h) = 0,
\:
\lambda_\delta\,\delta^* = 0.
\label{eq:slack}
\end{gathered}
\end{equation}
}

\paragraph{Role of the slack variable}  
From the second stationarity condition in \cref{eq:first-order}, and third condition in \cref{eq:slack}:
\begin{itemize}
    \item If $\delta^* > 0$, then $\lambda_\delta = 0$ and $\lambda_V = 2w_{\mathrm{clf}}\delta^*$.
    \item If $\delta^* = 0$, then $\lambda_\delta = \lambda_V$.
\end{itemize}
Thus $\delta^*$ is absorbed into $\lambda_V$, i.e., it only modulates the magnitude of the \ac{CLF} multiplier and does not directly change the control direction.

\paragraph{Stationarity in $\bm{u}$.}  
From $\partial_{\bm{u}}\mathcal{L}=0$:
\[
\bm{u}^* = -\tfrac12\bigl(\lambda_V \bm{a}_V - \lambda_h \bm{a}_h\bigr).
\]
This leads directly to the piecewise closed-form structure depending on which constraints are active:
\begin{equation} 
\bm{u}^* = 
\begin{cases} 
-\dfrac{1}{2}(\lambda_V \bm{a}_V - \lambda_h \bm{a}_h), & \lambda_V, \lambda_h >0,\\[6pt] 
-\dfrac{1}{2}\lambda_V \bm{a}_V, & \lambda_V>0, \lambda_h =0,\\[6pt] 
+\dfrac{1}{2}\lambda_h \bm{a}_h, & \lambda_h >0 , \lambda_V =0,\\[6pt] 
\bm{0}, & \lambda_V, \lambda_h =0.
\end{cases} 
\label{eq:u_closed_relax}
\end{equation}
where $\lambda_V = 2w_{\mathrm{clf}}\delta^*$ and $\lambda_h \ge 0$ are determined by the complementary slackness conditions.

\paragraph{Solving for the multipliers}  
When both constraints are active, substituting \eqref{eq:u_closed_relax} into the active constraints yields the linear system
\begin{equation}
\begin{bmatrix}
G_{11} + 1/w_{\mathrm{clf}} & G_{12}\\
G_{21} & G_{22}
\end{bmatrix}
\begin{bmatrix}\lambda_V\\ \lambda_h\end{bmatrix}
= -2\begin{bmatrix}b_V\\ b_h\end{bmatrix},
\end{equation}
with
\begin{equation*}
G_{11}=\bm{a}_V^\top \bm{a}_V,\quad 
G_{22}=\bm{a}_h^\top \bm{a}_h,\quad
G_{12}=G_{21}=\bm{a}_V^\top \bm{a}_h.    
\end{equation*}

The additional term $1/w_{\mathrm{clf}}$ arises from eliminating the slack variable through $\lambda_V = 2w_{\mathrm{clf}}\delta^*$. Let
\(
\Delta = (G_{11}+1/w_{\mathrm{clf}})G_{22} - G_{12}^2.
\)
If $\Delta > 0$, the closed-form solution is
\begin{equation}
\begin{gathered}
\lambda_V = \frac{-2\bigl(G_{22}b_V - G_{12}b_h\bigr)}{\Delta},\\
\lambda_h = \frac{-2\bigl((G_{11}+1/w_{\mathrm{clf}})b_h - G_{12}b_V\bigr)}{\Delta}.    
\end{gathered}    
\end{equation}

Negative multipliers are clipped to zero to enforce dual feasibility, which corresponds to deactivating the corresponding constraint:
\begin{equation}
\begin{gathered}
\lambda_V = \max\Big(0,\,-\frac{2b_V}{G_{11}+1/w_{\mathrm{clf}}}\Big),\\
\lambda_h = \max\Big(0,\,-\frac{2b_h}{G_{22}}\Big).   
\end{gathered}
\end{equation}
Finally, the CLF slack is recovered as
$\delta^* = \lambda_V/(2w_{\mathrm{clf}})$.

\paragraph{Remarks}
The key difference from the hard-\ac{CLF}~\cite{xiaoabnet} case is the presence of the slack variable, which softens the convergence constraint but does not affect the control direction explicitly. Its effect is entirely captured through $\lambda_V$, and the resulting Gram system remains $2\times 2$ with a simple closed-form solution. The \ac{CBF} is always strictly enforced, ensuring safety, while the \ac{CLF} can be relaxed to maintain feasibility near constraint boundaries. Moreover, the \ac{LSE} aggregation ensures smooth transitions when multiple obstacles are present, avoiding discontinuous active-set switching.

\begin{figure}[t]
    \centering
    \begin{subfigure}[b]{0.48\columnwidth}
        \centering
        \includegraphics[width=\linewidth]{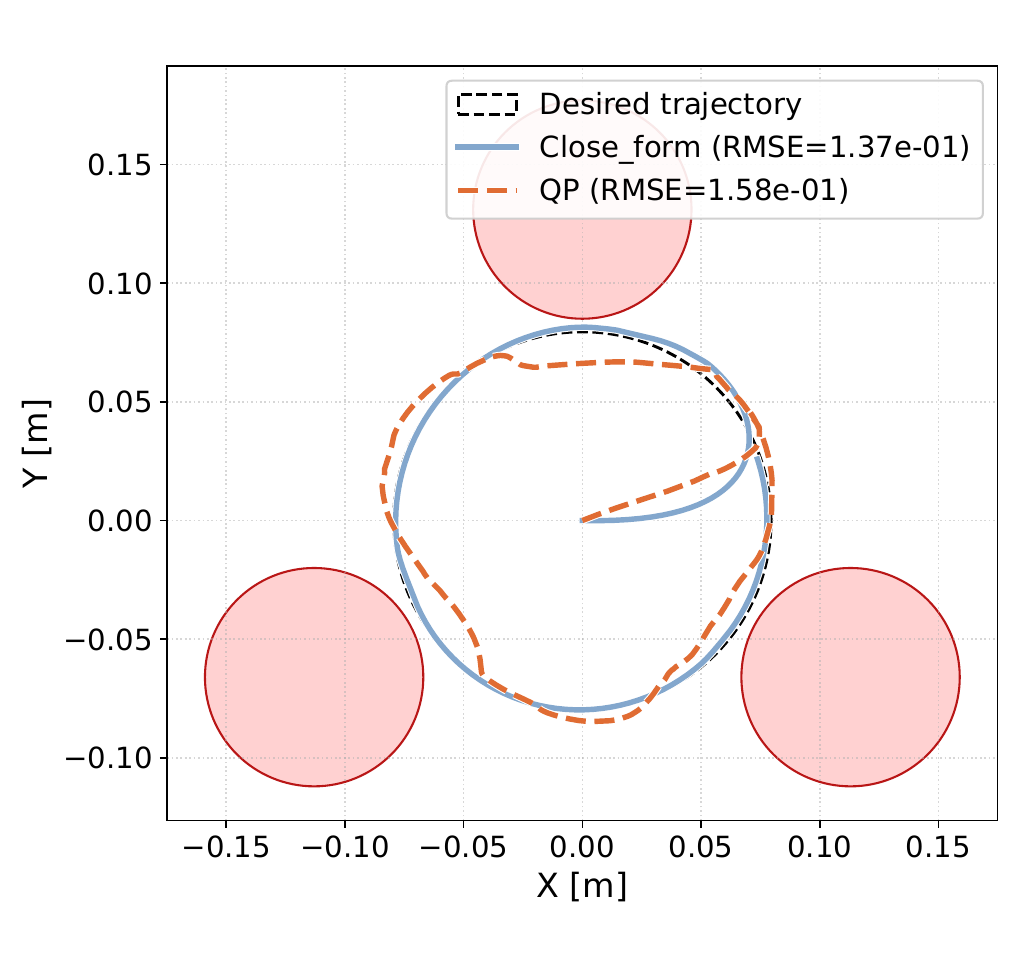}
        \caption{Speed: \SI{3}{RPM}, obstacles placed outside the trajectory}
        \label{fig:NC_20_loose}
    \end{subfigure}
    \hfill
    \begin{subfigure}[b]{0.48\columnwidth}
        \centering
        \includegraphics[width=\linewidth]{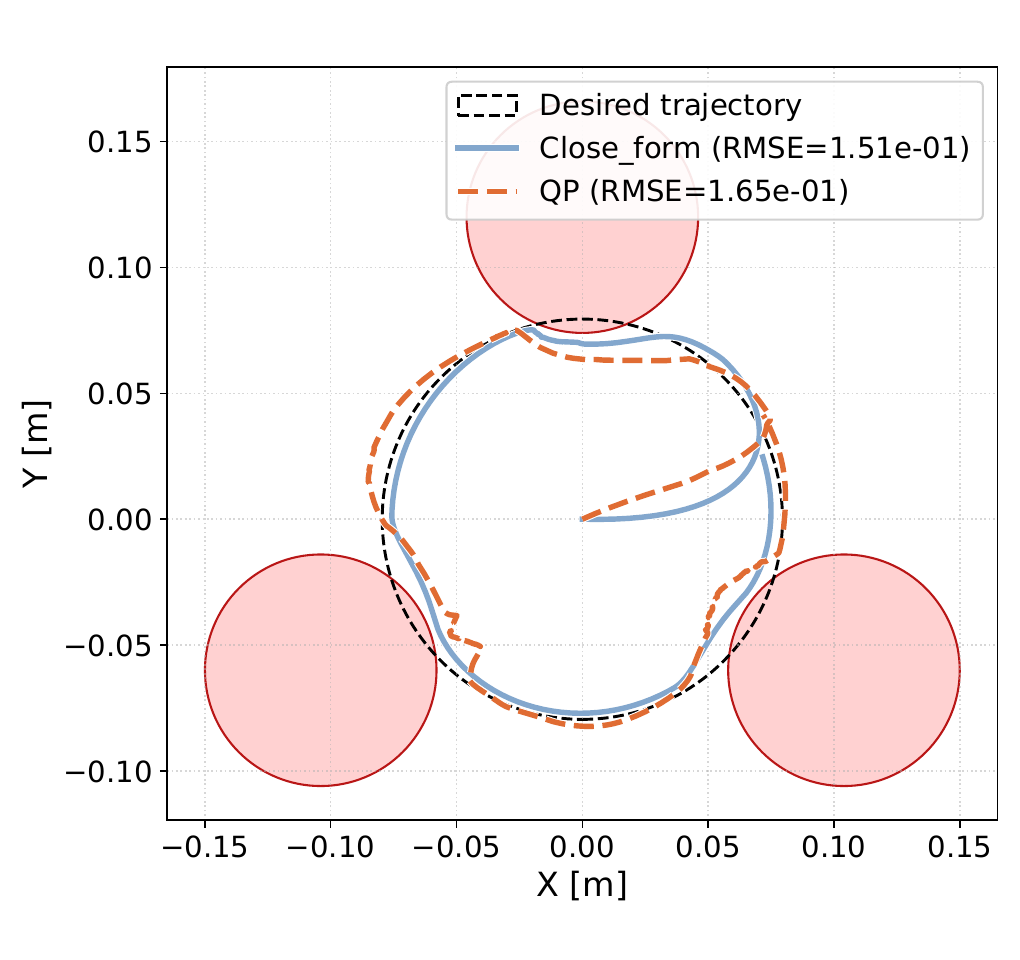}
        \caption{Speed: \SI{3}{RPM}, obstacles placed inside the trajectory}
        \label{fig:NC_20_tight}
    \end{subfigure}
    \begin{subfigure}[b]{0.48\columnwidth}
        \centering
        \includegraphics[width=\linewidth]{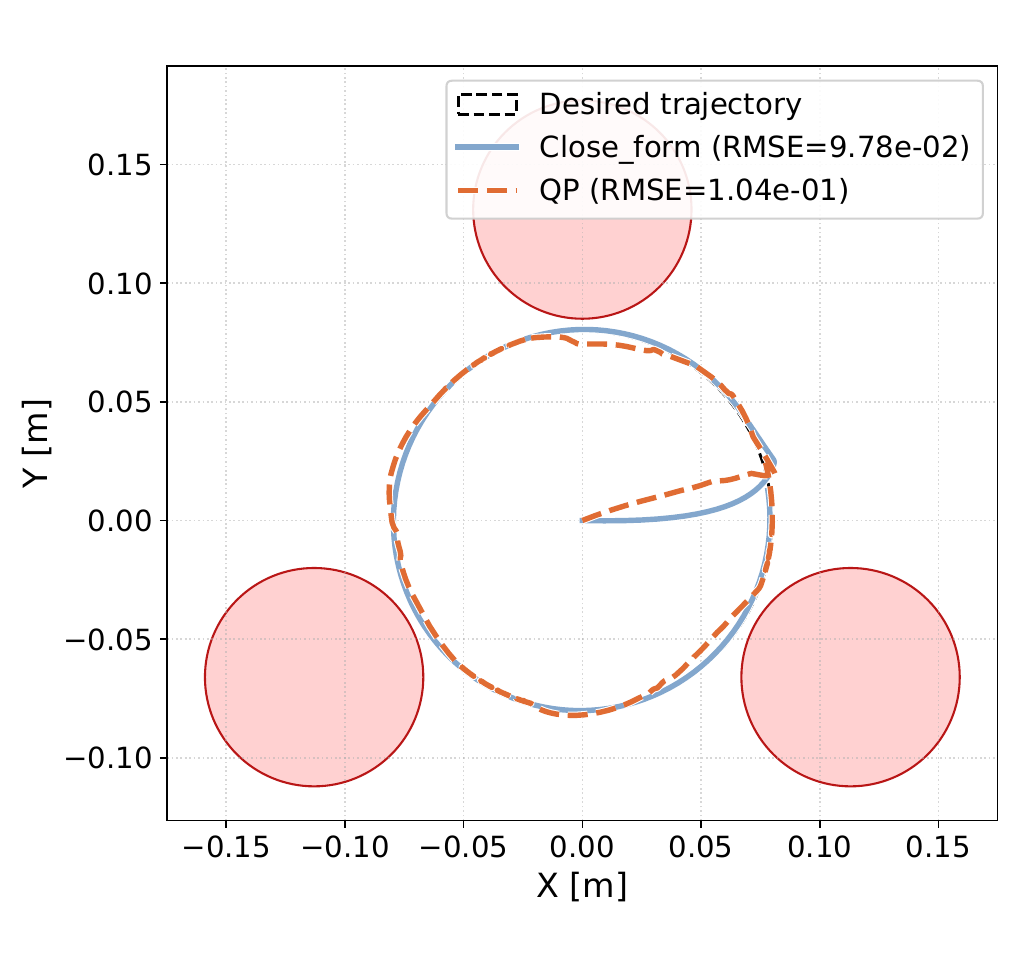}
        \caption{Speed: \SI{1.5}{RPM}, obstacles placed outside the trajectory}
        \label{fig:NC_40_loose}
    \end{subfigure}
    \hfill
    \begin{subfigure}[b]{0.48\columnwidth}
        \centering
        \includegraphics[width=\linewidth]{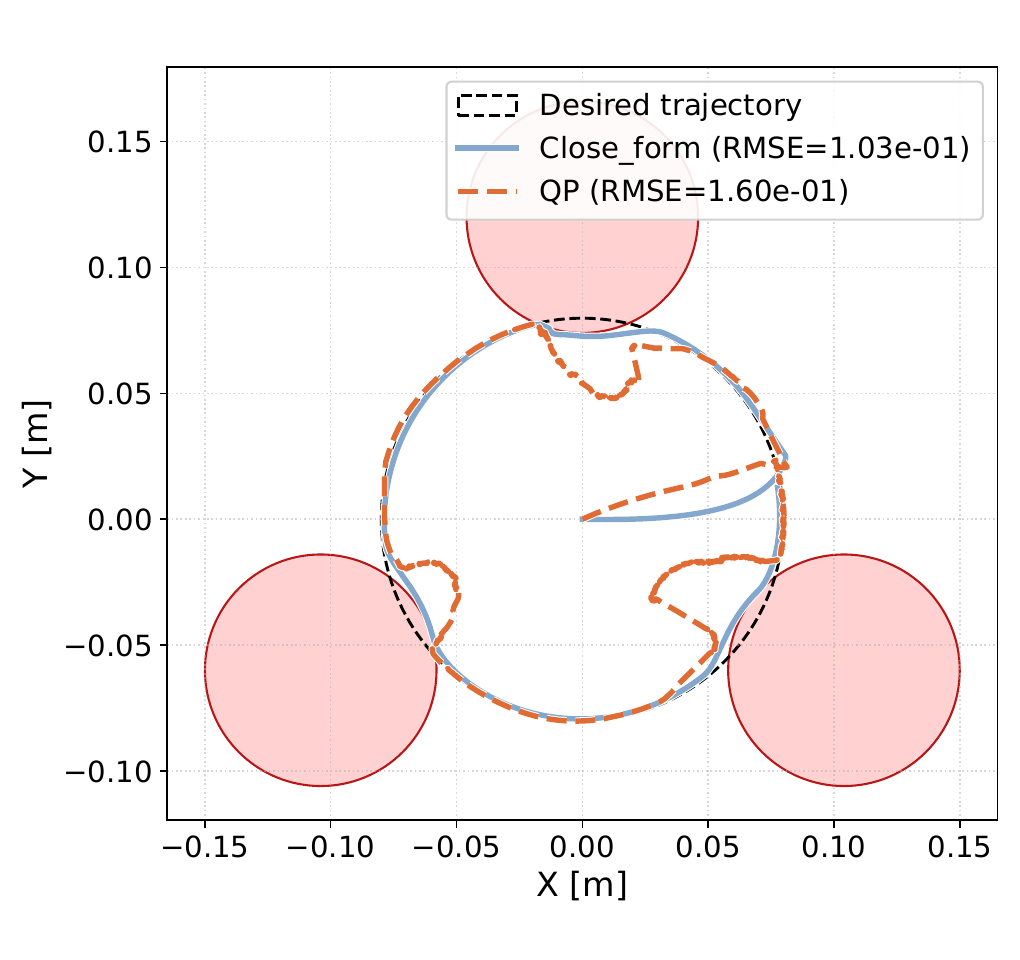}
        \caption{Speed: \SI{1.5}{RPM}, obstacles placed inside the trajectory}
        \label{fig:NC_40_tight}
    \end{subfigure}
        \caption{Simulation results. 
        Circular trajectory tracking pattern in 2D. 
        The trajectories represent the motion path of the soft robot tip center, while the red-shaded region indicates the collision area. 
        The RMSE of each case is shown in the legend, consistently demonstrating that the closed-form controller achieves more accurate and smoother trajectories than the baseline QP controller.
    }
    \label{fig:Comparison_Simulation}
\end{figure}

\section{Simulation Experiments}
In this section, we thoroughly evaluate in simulation how the proposed controller scales with the number of safety constraints and DOFs of the soft robot, we benchmark the computational time and numerical accuracy against a standard QP solver, and we perform setpoint regulation and trajectory tracking experiments, where we compare the behavior of the closed-loop system under the proposed closed-form CLF-CBF controller against an RRT* baseline strategy and, again, the same CLF-CBF solved using a standard QP solver.

\subsection{Simulation Setup}
We build our closed-form CLF-CBF controller on the \texttt{CBFpy}~\cite{morton2025oscbf} package that offers an intuitive user interface and high-performance auto-differentiation for CLF-CBF controller computation based on JAX. Our simulations consider an $N$-segment \ac{PCS}~\cite{renda2018discrete} soft robot operating in 3D space, approximated with $20$ spheres per segment for collision detection, and implemented in a differentiable fashion in the \texttt{SoRoMoX}~\cite{soromox2026} package. The soft robot has a total length of $L = \SI{0.3}{m}$ and a backbone radius of \SI{0.036}{m}. We assume that three tendons terminate at the end of each segment (e.g., $m=6$ for $N=2$) and that they are symmetrically distributed along the cross-section with $\phi_j = \frac{2}{3} \pi \, j$. The kinematic closed-loop ODE is integrated with a numerical solver implementing Tsitouras' 5/4 method with time step $\delta t = \SI{1}{ms}$. The safety margin, $d_{\mathrm{safe}}$ is set at \SI{0}{m}. We do not consider input constraints.

\subsection{Numerical Validation and Benchmarking of Closed-Form Solution against QP Solver}
First, we focus on verifying numerically the closed-form solution with respect to solving the \ac{QP} via the primal-dual interior point method.
We benchmark a two-constraint \ac{QP}
\begin{gather*}
\min_{\bm{u}} \ \|\bm{u}\|_2^2
\quad \text{s.t.}\quad A\bm{u} \le b,\qquad A\in\mathbb{R}^{2\times m},
\end{gather*}
where the slack variable $\delta = 0$:

An \emph{closed-form} solution (derived from KKT conditions) is compared against a JIT-compiled primal-dual interior point solver for convex \acp{QP} (\texttt{qpax}~\cite{tracy2024differentiability}) that we regard as a numerical reference.  
All methods are implemented in \texttt{JAX} and evaluated on CPU with double-precision floating point operations (\texttt{float64}).  
Each measured timing is averaged over $\si{10000}$ calls after JIT warm-up.

\begin{table}[tb]
\caption{Runtime and numerical accuracy for the two-constraint QP. The Closed-form method denotes the closed-form solution derived from KKT conditions, and the QP solver refers to the JIT-compiled numerical optimizer (\texttt{qpax}). 
Metrics are defined as follows: 
$\|\bm{u}_c-\bm{u}_{\text{qp}}\|_\infty$ and $\|\bm{u}_c-\bm{u}_{\text{qp}}\|_2$ measure the maximum and Euclidean norm of the solution difference, 
$|f(\bm{u}_c)-f(\bm{u}_{\text{qp}})|$ is the objective gap, 
and $\max(A\bm{u}-b,0)$ quantifies constraint violation. 
All results show machine-precision agreement between the closed-form and numerical solutions, 
with the closed-form approach achieving over an order-of-magnitude faster runtime.}
\begin{tiny}
\label{tab:hi-noq-bench}
\setlength{\tabcolsep}{2.5pt}
\renewcommand{\arraystretch}{1.05}
\centering
\footnotesize
\begin{tabular}{@{}lcc@{}}
\toprule
\textbf{Methods} & Runtime [$\mu$s/call] & Rel. Speedup \\
\midrule
QP solver (\texttt{qpax}, JIT)   & 45.53 & $1$x (baseline) \\
Closed-form & 4.22 & $10.8$x faster \\
\bottomrule
\end{tabular}

\vspace{0.6em}

\begin{tabular}{@{}lcccc@{}}
\toprule
\textbf{Numerical Accuracy} & Mean & Med. & 95th & Max \\
\midrule
$\|\bm{u}_c - \bm{u}_{\text{qp}}\|_\infty$
  & $5.2{\times}10^{-11}$ & $2.7{\times}10^{-12}$ & $8.7{\times}10^{-11}$ & $5.8{\times}10^{-9}$ \\
$\|\bm{u}_c - \bm{u}_{\text{qp}}\|_2$
  & $6.5{\times}10^{-11}$ & $3.8{\times}10^{-12}$ & $1.5{\times}10^{-10}$ & $6.3{\times}10^{-9}$ \\
$|f(\bm{u}_c)-f(\bm{u}_{\text{qp}})|$
  & $9.6{\times}10^{-13}$ & $1.6{\times}10^{-13}$ & $6.2{\times}10^{-12}$ & $1.2{\times}10^{-11}$ \\
$\max(A\bm{u}_c-b,0)$
  & $6.9{\times}10^{-18}$ & $0$ & $5.6{\times}10^{-17}$ & $5.6{\times}10^{-17}$ \\
$\max(A\bm{u}_{\text{qp}}-b,0)$
  & $1.1{\times}10^{-13}$ & $0$ & $1.1{\times}10^{-12}$ & $1.7{\times}10^{-12}$ \\
\bottomrule
\end{tabular}
\end{tiny}
\end{table}

The closed-form method achieves machine-precision agreement with the numerical solver while being more than an order of magnitude faster in steady-state runtime.  
No constraint violation is observed, and all $200/200$ active sets match exactly at tolerance $\tau=10^{-6}$, confirming the correctness and robustness of the closed-form solution.

\begin{figure}[htbp]
    \centering
    \subcaptionbox{Approximation error between the discretized and continuous robot body, 
        measured using the symmetric Hausdorff distance as the number of 
        discretized spheres $N_{\mathrm{res}}$ increases.%
        \label{fig:shape_error_vs_resolution}}[0.48\linewidth]{
        \includegraphics[width=\linewidth]{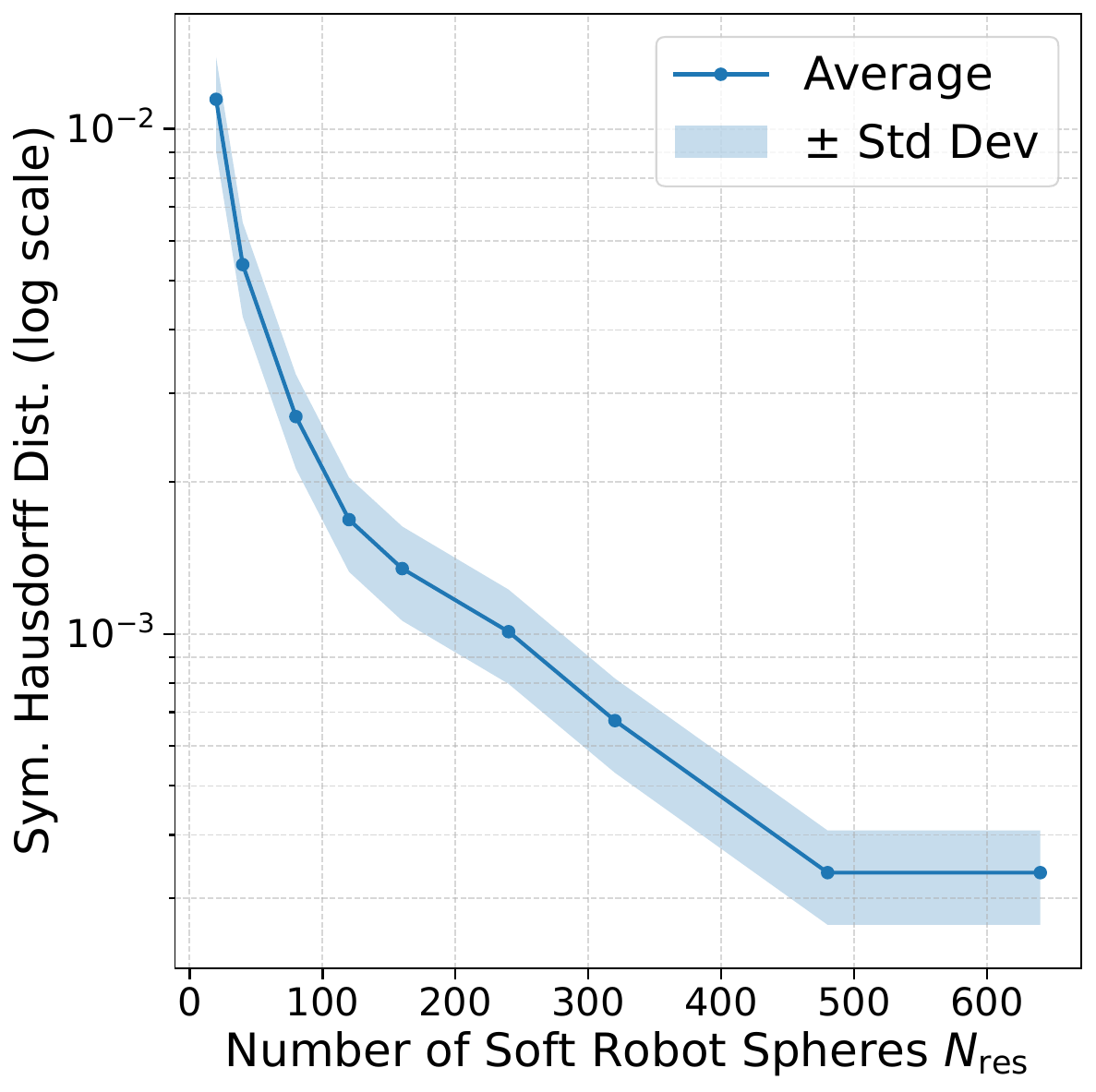}
    }
    \hfill
    \subcaptionbox{ Average computation time of the controller with increasing resolution.
        Closed-form CLF–CBF QP is significantly faster than the generic QP solver.%
        \label{fig:efficiency_vs_pps}}[0.48\linewidth]{
        \includegraphics[width=\linewidth]{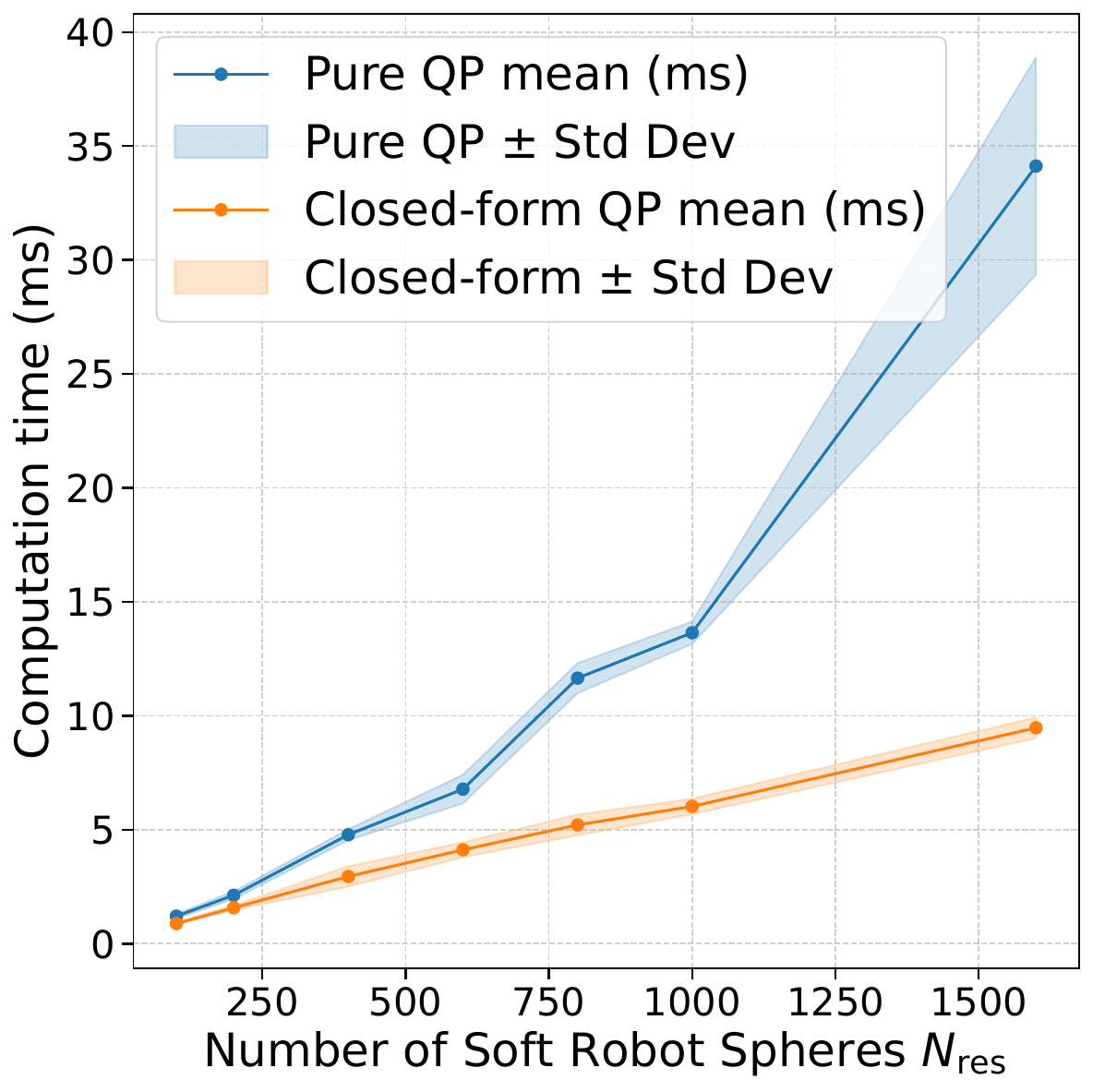}
    }\hfill
    
    \caption{Quantitative evaluation of body discretization and controller efficiency. 
        (a) Increasing $N_{\mathrm{res}}$ improves geometric fidelity, 
        with rapidly diminishing Hausdorff error after $\sim 400$ spheres. 
        (b) The closed-form CLF–CBF solver maintains low computation time 
        even at high discretization resolutions, in contrast to the generic QP solver.}
    \label{fig:discretization_performance}
\end{figure}

\subsection{Geometrical Approximation of Soft Robot Body and Scaling of Computational Time with Resolution}

We evaluate how the discretization resolution of the soft robot affects both geometric fidelity and controller computational efficiency, as reported in \Cref{fig:discretization_performance}. Specifically, we first simulate a trajectory for the robot and randomly select configurations for evaluation. Subsequently, we measure the mismatch between the actual soft robot geometry and the geometry for collision detection, approximated with a chain of spheres, using the symmetric Hausdorff distance. Indeed, it is shown how increasing the resolution $N_{\mathrm{res}}$ improves the accuracy of the geometric approximation of the robot’s shape and decreases the symmetric Hausdorff distance.
Furthermore, we evaluate the computational time necessary for one controller evaluation and compare solving the QP online using the \texttt{CBFpy}~\cite{morton2025oscbf} with our proposed closed-form solution. We observe that the closed-form solution maintains significantly lower computational time as $N_{\mathrm{res}}$ increases, demonstrating its advantages with regard to scalability and efficiency when increasing the number of safety constraints.

\subsection{Baseline: RRT* in Configuration-Space}
As we discussed in the introduction, currently, sampling-based path planning methods, such as RRT variants, are the most widely established approaches for obstacle avoidance with continuum soft robots~\cite{torres2015motion, kuntz2017motion, meng2022rrt, li2024s, lang2025obstacle}.
Therefore, we adapt as a baseline an RRT*~\cite{karaman2011sampling} path planner that operates in configuration-space, checks for collisions, and returns a sequence of kinematic configurations that move the robot towards the target end-effector position and that are tracked by a kinematic low-level controller. Please note that we had to neglect shear and twist strains to keep the dimensionality of the planning space manageable. 
More sophisticated baselines (e.g., MPC-based controllers) could improve tracking performance, but typically lack formal safety guarantees or require online optimization.

Our implementation is similar to the method proposed by Kuntz \textit{et al.}~\cite{kuntz2017motion}, apart from that we use RRT*~\cite{karaman2011sampling} instead of vanilla RRT~\cite{lavalle2001rapidly}.
Specifically, we implement a RRT* variant that operates in an affine‑normalized configuration space: the robot’s physical joint/strain bounds are mapped to $[0,1]^{n_\mathrm{q}}$ so that sampling, nearest‑neighbor queries, and distances are well‑conditioned across heterogeneous units. Sampling is uniform in this normalized space, edges are grown with a fixed step parameter ($q=0.02$), and collisions are screened along each edge at a finer stride ($r=0.01$) by evaluating a control barrier function $b(\bm{q}) \ge 0$ on denormalized configurations; Goal progress is driven via the \ac{CLF} $V(\bm{q})$; planning terminates once $V<\SI{0.01}{m}$. We enable rewiring ($k=32$ neighbors) to improve path cost, and cap iterations at $\SI{20480}{}$ samples. If no node satisfies the threshold within the budget, we return the path to the node that attained the smallest $V(\bm{q})$ (closest approach).

Subsequently, this sequence of configurations is tracked by a kinematic controller that drives the system towards the next configuration setpoint $\bm{q}^\mathrm{d} \in \mathbb{R}^{n_\mathrm{q}}$
\begin{equation}
    \bm{u}(t) = \dot{\bm{\ell}}(t) = \bm{J}_{\bm{\ell}}(\bm{q}) \, \bm{K}_\mathrm{p} \left ( \bm{q}^\mathrm{d}(t) - \bm{q}(t) \right ),
\end{equation}
where $\bm{K}_\mathrm{p} \succ 0 \in \mathbb{R}^{n_\mathrm{q} \times n_\mathrm{q}}$ determines the proportional feedback gains. In our simulations, we use $\bm{K}_\mathrm{p} = 2 \, \mathbb{I}_{n_\mathrm{q}}$.
As soon as the system has reached a tight region around the setpoint, or \SI{4}{s} have passed, the next setpoint from the planned path is provided to the controller.

Please note that a cascaded scheme—path planner plus safety-unaware low-level controller—exhibits safety gaps: collision freedom holds only under the planner’s heuristic of linear interpolation in configuration space between setpoints/nodes. In practice, the closed-loop motion deviates between setpoints, so collisions aren’t guaranteed to be avoided.

\subsection{Setpoint Regulation}
We design a setpoint regulation scenario, with shearing strain active, in which a two-segment ($N=2$) soft robot must navigate in an environment with multiple obstacles while attempting to reach a fixed target at $(0.10, 0.05, 0.32)$. The spherical obstacles are placed at
$(0.10, 0.08, 0.24)$~m, $(0.12, 0.06, 0.32)$~m, and $(0.04, 0.055, 0.20)$~m, and each has a radius of \SI{0.02}{m}. 
In this scenario, the controller must ensure that the robot approaches the target as closely as possible while strictly maintaining a minimum clearance from all obstacles. \cref{fig:simulation_setpoint_regulation_clf_cbf_values} plots the \ac{CLF} value $V(\bm{q}(t))$ and the minimum of the barrier functions $\min_{i \in \mathbb{N}_\mathrm{res}, j \in \mathbb{N}_\mathrm{obs}} b_{i,j}(\bm{q}(t))$ over the simulation time both for the RRT* baseline and the proposed closed-form CLF-CBF controller.
\cref{fig:simulation_setpoint_regulation_stills} contains a sequence of stills illustrating how the controller enables the robot to approach the target while maintaining safe distances from the obstacles, and \cref{fig:simulation_setpoint_regulation_rrtstar_ee_time_evolution} plots the time evolution of the end effector position for the RRT* baseline.
These simulation results motivate several advantages of the proposed CLF-CBF controller against established obstacle avoidance strategies relying on a cascaded scheme of a sampling-based planner with a low-level controller: (i) the proposed strategy exhibits a smoother and monotone convergence towards the goal as it does not rely on discrete intermediate setpoints, (ii) the RRT* baseline identifies because of the limited sampling budget as less optimal path with a higher final \ac{CLF} value, and (iii) while a thorough comparison of the computational time between the proposed CLF-CBF controller and the RRT* baseline is not straightforward (e.g., there exists a tradeoff between planning time and optimality; RRT* planning needs to be done once before the start of the motion for a given set of obstacles and a given goal, while the proposed controller does not rely on \emph{a priori} planning) and, therefore, out of scope for this paper, we observe that the RRT* planning takes \SI{244}{s} on an 12th-gen Intel i9 CPU with 24 cores and \SI{64}{GB} of RAM, which is orders of magnitudes more expensive than the proposed closed-form strategy.

\begin{figure}
    \centering
     \centering
    \subcaptionbox{RRT* + Low-Level Controller (Baseline)\label{fig:simulation_setpoint_regulation_clf_cbf_values:rrt_star}}[0.48\linewidth]{
        \includegraphics[width=\linewidth, trim={5 5 5 5}]{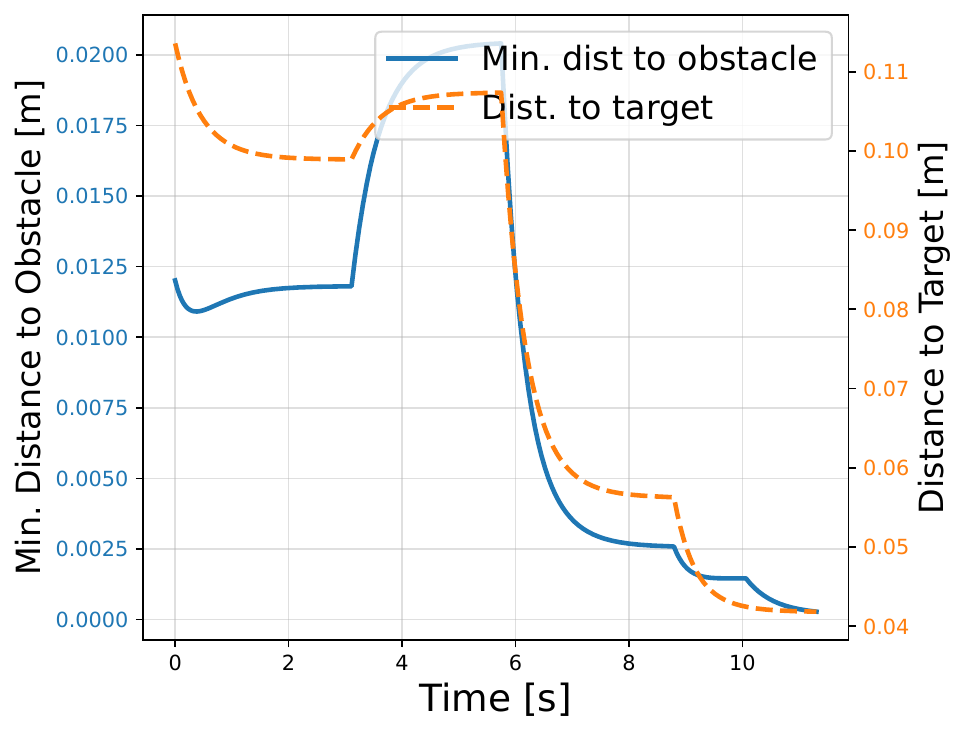}
    }
    \hfill
    \subcaptionbox{Closed-Form CLF-CBF Controller (Proposed)\label{fig:simulation_setpoint_regulation_clf_cbf_values:clfcbf}}[0.48\linewidth]{
        \includegraphics[width=\linewidth, trim={5 5 5 5}]{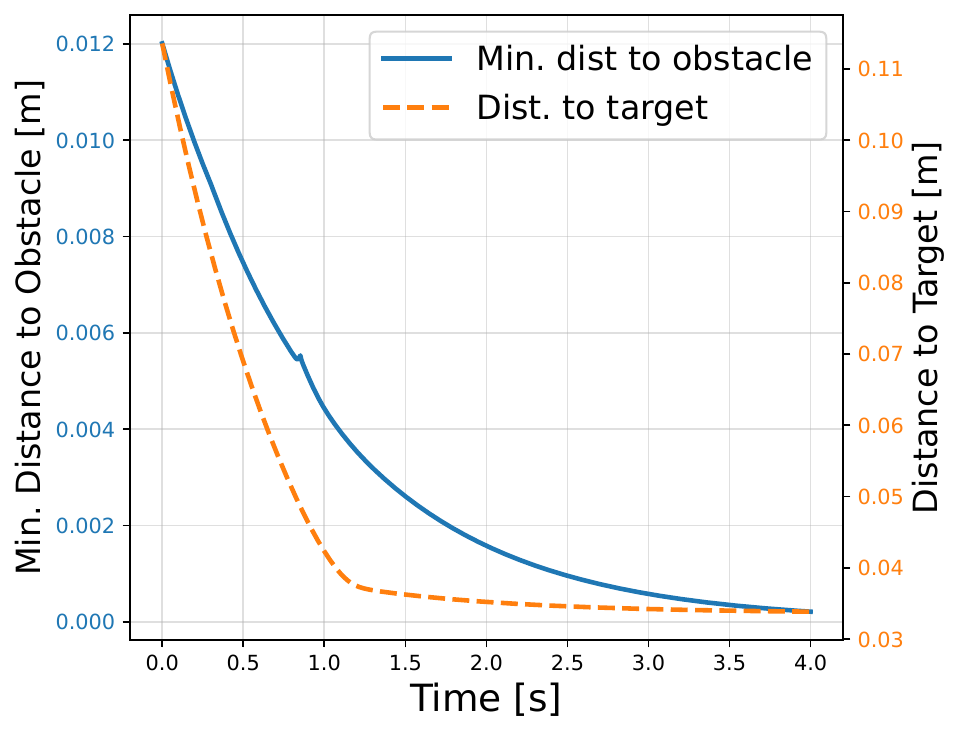}
    }\hfill
    \caption{Setpoint regulation simulation results for the distance to the target (CLF objective) and to the obstacles (CBF constraints). \textbf{Left.} Behavior of a low-level controller tracking a sequence of configuration setpoints provided by an RRT* obstacle avoidance planner. \textbf{Right.} Behavior of the closed-loop system under the proposed closed-form CLF-CBF controller.
    The distance to the target does not reach absolute zero due to the presence of nearby obstacles.}
    \label{fig:simulation_setpoint_regulation_clf_cbf_values}
\end{figure}

\begin{figure}[htbp]
    \centering
    \includegraphics[width=0.95\linewidth]{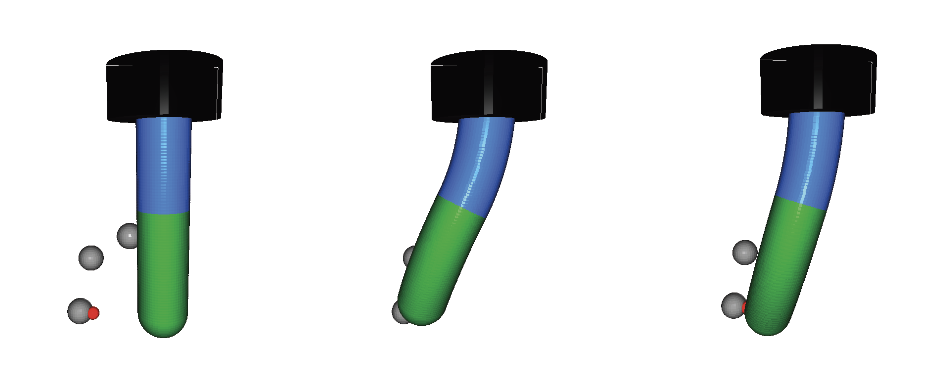}
    \caption{Sequence of stills of the setpoint regulation simulations visualizing the motion of a soft robot consisting of two segments (blue and green) under the closed-form CLF–CBF controller, with time progressing from left to right. The controller steers the robot toward the target denoted with a red dot while preserving a safe distance from surrounding obstacles (grey).}
    \label{fig:simulation_setpoint_regulation_stills}
\end{figure}

\begin{figure}[htbp]
    \centering
    \includegraphics[width=0.95\linewidth]{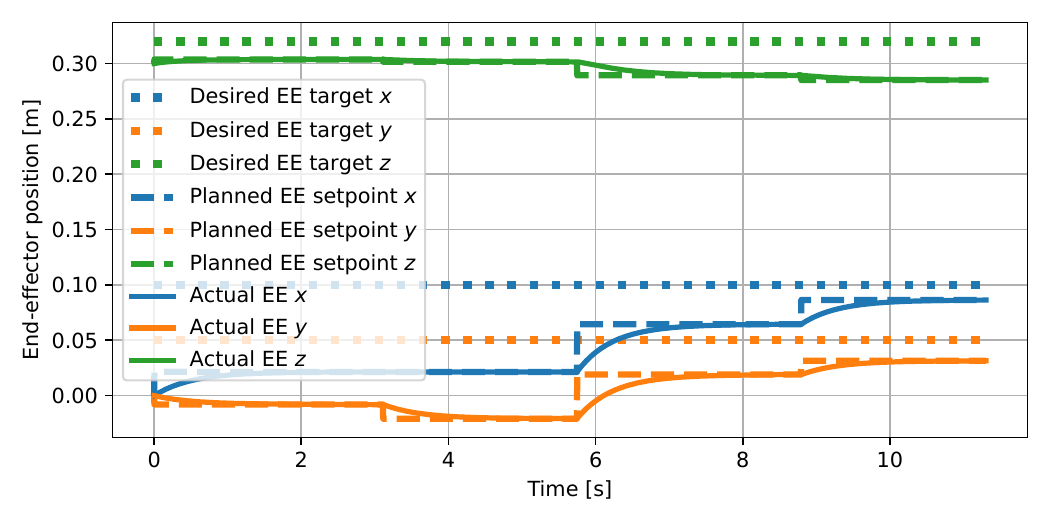}
    \caption{Setpoint regulation simulation results for the time evolution of the end-effector (EE) position under the RRT* baseline collision avoidance strategy. The dotted lines represent the desired EE targets (i.e., the task encoded in the \ac{CLF}), the dashed lines the sequence of the setpoints provided by the RRT* planner that serve as intermediate goals for the low-level kinematic controller, and the solid lines the actual motion of the closed-loop system.}
    \label{fig:simulation_setpoint_regulation_rrtstar_ee_time_evolution}
\end{figure}

\subsection{Trajectory Tracking}
Furthermore, we evaluate the performance of the proposed closed-form controller against the standard \ac{QP}-based controller on a circular trajectory tracking task with $N=2$ and while neglecting shear strains. The inputs are clipped in \SI{\pm 0.02}{m/s}. A reference target lying on the circular trajectory was provided at each time step, and the reference variation was adjusted in different cases to examine whether the controllers could still track the trajectory under faster reference changes. The slack variable $\delta$ is set to zero. A set of spherical obstacles with radii of \SI{0.01}{m},  was placed on the same plane as the robot’s target $z$-plane to introduce controlled collision scenarios and in 120 degrees apart. This setup allows us to assess the controller’s ability to maintain accurate trajectory tracking while ensuring safety under different obstacle configurations.

\Cref{fig:Comparison_Simulation} visualizes the trajectories for both controllers with respect to the circular reference. Notably, the closed-form controller consistently follows the trajectory more accurately and smoothly than the standard \ac{QP}-based controller. The closed-form controller consistently follows the trajectory more accurately and smoothly than the standard QP-based controller. This is because the closed-form solution yields a continuous and closed-form control law, which avoids the discontinuities induced by active-set switching in QP solvers. By embedding the safety constraints directly into the control law, the closed-form controller produces smoother control inputs and reduces abrupt trajectory deviations, particularly near obstacles.

\begin{figure}[htbp]
    \centering
    \includegraphics[width=\columnwidth]{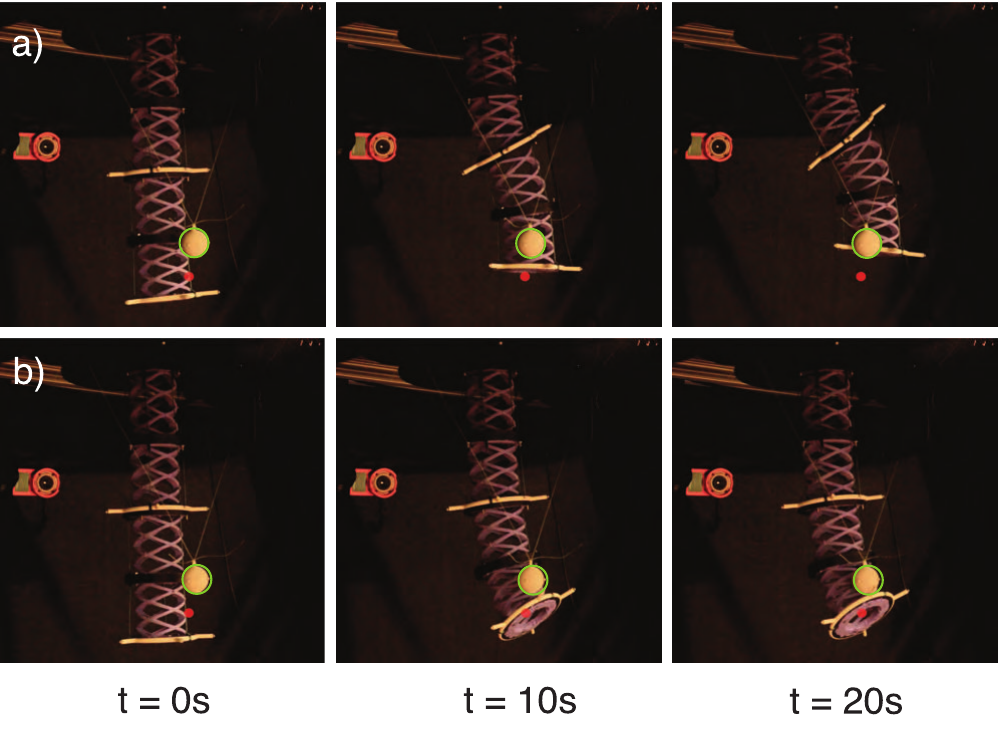}
    \caption{Experimental results. Sequence of stills showing the system evolution under: a) the closed-form CLF–CBF controller, b) a QP CLF controller. The target position is highlighted in red, and the obstacle is outlined in green. Consistently, the closed-form CLF–CBF controller avoids the obstacle entirely, whereas the pure CLF controller reaches the target but violates the safety constraints.}
    \label{fig:stills}
\end{figure}

\section{Laboratory Experiments}
We further assessed the proposed closed-form CLF–CBF controller through a representative setpoint tracking experiment on a tendon-driven soft robotic arm in a 3D laboratory environment.

\subsection{Tendon-Driven Soft Robotic Platform}
\label{sec:hardware_setup}
A tendon-driven, two-segment soft robotic arm is employed for the real-time control experiments. The soft robot design is based on a previous work~\cite{patterson2025design}. Each robot segment is \SI{0.15}{m} long, and \SI{0.036}{m} in radius. Reflective markers are placed at the base and at the tip of each segment, enabling feedback of Cartesian poses through five OptiTrack motion capture cameras. Subsequently, we perform inverse kinematics in closed-form via the SE(3) $\operatorname{Log}$ on the relative transformation between the proximal and distal end of each segment to gather the current configuration of the soft robot $\bm{q}$~\cite{soromox2025}.

Each segment contains three tendons evenly distributed at $120^\circ$ intervals around the backbone. Between the two segments, the tendon routing is offset by $60^\circ$ to improve dexterity and reduce actuation coupling. The tendons are actuated by DYNAMIXEL XM430-W350-R motors operated in velocity control mode. Each motor is equipped with a spool of radius \SI{0.0045}{m}, and constrained in rotation speed to \SI{\pm 16}{RPM}. The controller is evaluated at a rate of \SI{100}{Hz} and communicates with the actuators in real time through \texttt{ROS}, allowing for smooth execution of the CLF–CBF–based controllers and dynamic trajectories.

\begin{figure}[htbp]
    \centering
    \subcaptionbox{Minimum target distance variation with respect to time.\label{fig:target_distance_real_world}}[0.48\linewidth]{
        \includegraphics[width=\linewidth]{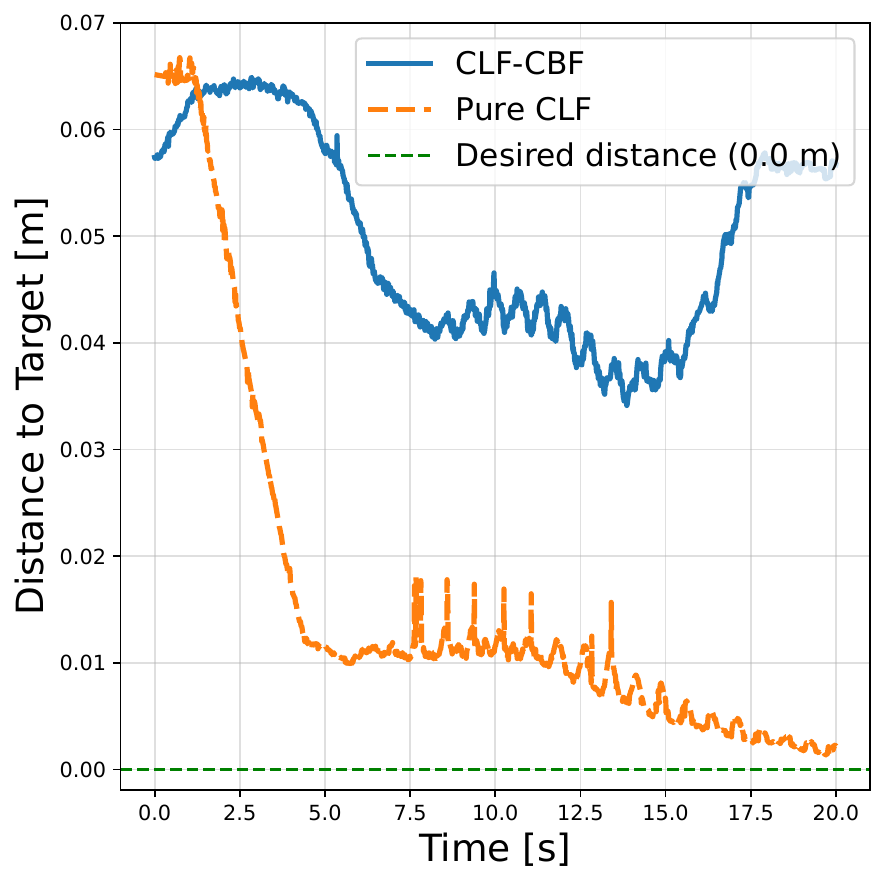}
    }\hfill
    \subcaptionbox{Minimum distance for the whole robotic body to the obstacle.\label{fig:obs_distance_real_world}}[0.48\linewidth]{
        \includegraphics[width=\linewidth]{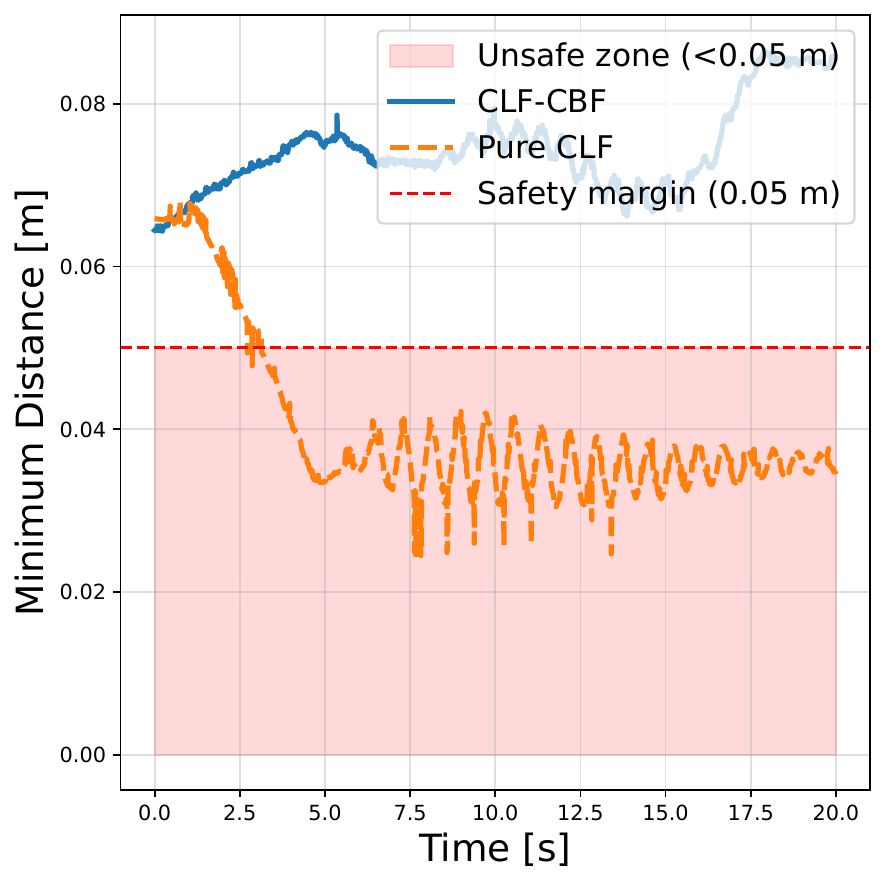}
    }
    \caption{Experimental results. Comparison of target and obstacle distances for closed-form CLF-CBF controller, and a pure QP CLF controller during the hardware experiment.}
    \label{fig:distance_real_world}
\end{figure}

\subsection{Results}
We conducted closed-loop control experiments to qualitatively compare the behavior of the proposed closed-form CLF–CBF controller with that of a pure CLF controller, which does not enforce safety constraints. 
We employ the same software and control packages as in the simulation experiments while neglecting shear strains. For the real robot, the soft body was discretized into $40$ sections to achieve a sufficiently high spatial resolution for collision avoidance, and the slack variable~$\delta$ is chosen as $0$.
In this experiment, the desired target lies within the influence region of the obstacle, rendering exact convergence infeasible under the enforced safety margin. The controller therefore converges to the closest safe equilibrium.

Figure~\ref{fig:stills} shows representative stills of the robot’s motion under both controllers. Under the pure CLF controller, the robot quickly reaches the target but violates the safety distance by colliding with the obstacle. In contrast, the closed-form CLF–CBF controller maintains a safe distance throughout the trajectory while achieving stable convergence.

\Cref{fig:target_distance_real_world} plots the distance to the target over time. The closed-form controller achieves smooth convergence without overshoot, whereas the pure CLF controller reaches the target faster but with no safety enforcement. \Cref{fig:obs_distance_real_world} further depicts the minimum distance to the obstacle. While the pure \ac{CLF} controller enters the unsafe region, the CLF–CBF controller consistently maintains the predefined \SI{0.05}{m} safety margin.

These results provide a qualitative illustration of the controller’s ability to enforce safety constraints in a representative laboratory scenario. This experiment highlights the effectiveness of the closed-form CLF–CBF controller in maintaining safety without sacrificing nominal tracking performance.
\section{Conclusion}
This paper presents a closed-form CLF–CBF controller framework for soft robots, enabling real-time safe control without the need for online optimization. By deriving an explicit solution to the CLF–CBF quadratic program, the proposed method guarantees safety constraint satisfaction while maintaining low computational overhead. The controller is built upon a differentiable \ac{PCS} kinematics model and safety constraints evaluated along the entire soft body, ensuring whole-body safety in cluttered environments. Simulation and laboratory experiments demonstrate that the closed-form controller achieves better setpoint regulation behavior compared to traditional sampling-based planning approaches at two order of magnitudes less computational expense and better tracking performance compared to online optimization-based methods for solving the CLF-CBF QP while strictly respecting safety constraints. 

The main limitation of the current approach is that it aggregates only a single constraint type—specifically, a collision-avoidance barrier—and cannot accommodate heterogeneous constraints (e.g., input bounds, energy budgets, or task-specific requirements). Future work will extend the closed-form CLF–CBF controller to support compositionality across multiple constraint classes and to incorporate hierarchical prioritization, enabling more expressive and flexible safety-critical control. Furthermore, while the current approach neglects the inertia of the soft robot and excels in quasi-static settings, future work could extend it to dynamic collision avoidance scenarios via the use of higher-order CLFs and CBFs~\cite{xiao2021high, wong2025contact}. Finally, we will also examine experimental settings that more closely reflect real-world deployment scenarios.

\bibliographystyle{IEEEtran}
\bibliography{main.bib}
\end{document}